\documentclass[11pt]{article}

\usepackage[letterpaper,margin=1in]{geometry}

\usepackage[T1]{fontenc}
\usepackage[utf8]{inputenc}
\usepackage{lmodern}
\usepackage[hyphens]{url}
\usepackage{amsmath}
\usepackage{amssymb}
\usepackage{amsthm}

\usepackage{multirow}
\usepackage[table]{xcolor}
\usepackage{graphicx}
\usepackage{booktabs}

\usepackage{natbib}
\usepackage{caption}

\usepackage{algorithm}
\usepackage{algorithmicx}
\usepackage{algpseudocode}

\usepackage{hyperref}
\hypersetup{
    colorlinks=true,
    linkcolor=blue!50!black,
    citecolor=blue!50!black,
    urlcolor=blue!50!black,
}

\usepackage{cleveref}

\theoremstyle{plain}
\newtheorem{proposition}{Proposition}
\newtheorem{corollary}{Corollary}
\newtheorem{lemma}{Lemma}

\crefname{proposition}{proposition}{propositions}
\Crefname{proposition}{Proposition}{Propositions}
\crefname{corollary}{corollary}{corollaries}
\Crefname{corollary}{Corollary}{Corollaries}
\crefname{lemma}{lemma}{lemmas}
\Crefname{lemma}{Lemma}{Lemmas}

\usepackage{xspace}
\newcommand{\JOT}{\textsc{JoT}\xspace}

\title{Just on Time: Token-Level Early Stopping for Diffusion Language Models}

\author{
    Zakhar Kohut\textsuperscript{*}, Severyn Shykula\textsuperscript{*}, Mykola Vysotskyi\textsuperscript{*}, Serhii Dmytryshyn, Dmytro Khamula, \\
    Michal Zakrzewski, Damian Rynczak, Jacek Ma{\l}ecki, Taras Rumezhak, Volodymyr Karpiv \\[1em]
    SoftServe Inc. \\[0.5em]
    {\small \textsuperscript{*}Equal contribution.}
}
\date{}

\begin{document}

\maketitle

\begin{abstract}
Diffusion language models generate text through iterative refinement, a process that is often computationally inefficient because many tokens reach stability long before the final denoising step. We introduce a training-free, token-level early stopping approach that identifies convergence independently at each position. Our method leverages lightweight signals derived from the model's predictions and local context to dynamically determine when individual tokens can be finalized. This yields adaptive per-token freezing without task-specific fine-tuning, substantially reducing the total number of diffusion steps required. Across diverse benchmarks, spanning mathematical reasoning, general question answering, and scientific understanding, our approach achieves substantial efficiency gains while preserving generation quality.
\end{abstract}

\section{Introduction}

Diffusion language models (DLMs) have emerged as a compelling alternative to autoregressive generation \citep{austin2023structureddenoisingdiffusionmodels}. By starting from a fully masked sequence and iteratively denoising, DLMs enable parallel token prediction and bidirectional context integration, achieving competitive performance on diverse tasks \citep{nie2025largelanguagediffusionmodels,ye2025dream7bdiffusionlarge}. However, \emph{decoding efficiency} remains a challenge: generation requires many refinement steps, yet many tokens converge to stable predictions well before the final step \citep{li2025diffusionlanguagemodelsknow}, leading to unnecessary computation.

We introduce \JOT (\textbf{J}ust \textbf{o}n \textbf{T}ime), a training-free method for per-token early stopping in DLMs. Rather than applying a global stopping criterion, \JOT monitors prediction confidence at each position independently and finalizes tokens once they exceed a spatially-adaptive threshold. This allows different positions to exit parallely at different steps, concentrating computation where it is needed.

We evaluate \JOT on Dream-7B-Instruct and LLaDA-8B-Instruct across four benchmarks: GSM8K, MMLU, HellaSwag, and HumanEval. Our experiments demonstrate that \JOT achieves favorable speed-quality trade-offs, providing up to $5.5\times$ speedup on GSM8K and $19.6\times$ on HumanEval while maintaining scores within ${\sim}3$ percentage points of full decoding on most benchmarks (the largest gap being $-3.1$ points on LLaDA HumanEval). Ablation studies confirm that both threshold selection and spatial modulation contribute to performance.

\paragraph{Contributions.} We summarize our main contributions as follows:
\begin{itemize}
    \item We propose \JOT, a training-free, per-token early stopping method for diffusion language models that adapts acceptance thresholds based on spatial proximity to resolved context.
    \item We conduct comprehensive experiments on Dream-7B and LLaDA-8B across four benchmarks, analyzing speed-quality trade-offs against existing early-exit methods.
    \item We provide detailed ablation studies isolating the effects of threshold selection and spatial modulation, offering practical guidance for hyperparameter configuration.
\end{itemize}

\section{Related Work}
\label{sec:related}

\paragraph{Diffusion Language Models.} Diffusion models have achieved remarkable success in continuous domains such as image and audio generation \citep{ho2020denoisingdiffusionprobabilisticmodels, song2021scorebasedgenerativemodelingstochastic}, and extending them to discrete text has followed two paths: embedding-based methods that project tokens into continuous space before applying standard diffusion \citep{li2022diffusionlmimprovescontrollabletext, gong2023diffuseqsequencesequencetext}, and discrete diffusion over the vocabulary directly. D3PM \citep{austin2023structureddenoisingdiffusionmodels} established a foundational discrete framework using Markov chains with learnable transition matrices, and SEDD \citep{lou2024discretediffusionmodelingestimating} connected score-based and likelihood-based training for discrete sequences.
Among discrete formulations, masked (absorbing-state) diffusion---where the forward process replaces tokens with a mask token and the reverse process predicts the originals---has proven particularly effective. MDLM \citep{sahoo2024simpleeffectivemaskeddiffusion} derived a simplified objective connecting to BERT-style masked language modeling \citep{devlin2019bertpretrainingdeepbidirectional}. LLaDA \citep{nie2025largelanguagediffusionmodels} scaled this paradigm to 8B parameters with competitive autoregressive performance, and Dream 7B \citep{ye2025dream7bdiffusionlarge} further advanced the state of the art with autoregressive initialization and context-adaptive noise rescheduling.

\paragraph{Early Stopping for Diffusion Decoding.} Recent work has observed that predictions often stabilize well before the final diffusion step, motivating early-exit strategies. Prophet \citep{li2025diffusionlanguagemodelsknow} documents this ``early answer convergence'' and proposes a \emph{global} early-commit rule: once the top-2 confidence gap exceeds a threshold, all remaining masked positions are finalized simultaneously. In contrast, KLASS \citep{kim2025klassklguidedfastinference} operates at the \emph{per-token} level, using KL divergence between consecutive step distributions to identify individually stable predictions and unmask them in parallel. Orthogonal to step reduction, caching methods, such as dKV-Cache or D2F, target per-step latency \citep{ma2025dkvcachecachediffusionlanguage, wang2025diffusionllmsfasterthanarinference}.

In the terminology of parallel-decoding work, \JOT performs \emph{confidence-gated parallel unmasking}: like Fast-dLLM \citep{wu2025fastdllmtrainingfreeaccelerationdiffusion}, it finalizes masked positions whose confidence clears a gate, but with a spatially adaptive, temperature-invariant gate rather than a fixed probability threshold. Throughout the paper we use \emph{early exit} for this mechanism: a token exits iterative refinement as soon as its prediction is deemed stable. This is distinct from methods that stop computation on \emph{already-decoded} tokens \citep{oba2025decodedtokens}, which target a different inefficiency: redundant computation on finalized positions rather than premature refinement of unfinalized ones.
\section{Preliminaries}
\label{sec:preliminary}
\subsection{Discrete Diffusion Language Models}

Consider a vocabulary $\mathcal{V}$ of size $K$ augmented with a mask token $\texttt{[M]}$. A sequence of length $L$ is denoted $\mathbf{x} = (x^1, \ldots, x^L)$ with $x^i \in \mathcal{V}$. 

Autoregressive language models factorize the joint distribution as $p(\mathbf{x}) = \prod_{i=1}^{L} p(x^i \mid x^{<i})$ and generate tokens sequentially from left to right. In contrast, discrete diffusion models define a generative process through the interplay of a forward corruption process and a learned reverse denoising process, enabling parallel refinement of all positions simultaneously.

\paragraph{Forward process.} The forward process progressively corrupts a clean sequence $\mathbf{x}_0$ by independently replacing tokens with the mask token. Let $t \in [0, 1]$ denote the continuous noise level, where $t = 0$ corresponds to clean data and $t = 1$ to the fully masked state. The forward marginal factorizes over positions as $q(\mathbf{x}_t \mid \mathbf{x}_0) = \prod_{i=1}^{L} q(x_t^i \mid x_0^i)$, where
\begin{equation}
    q(x_t^i \mid x_0^i) = \alpha_t \, \delta(x_t^i = x_0^i) + (1 - \alpha_t) \, \delta(x_t^i = \texttt{[M]}).
\end{equation}
Here, $\alpha_t \in [0, 1]$ is a monotonically decreasing noise schedule with $\alpha_0 \approx 1$ and $\alpha_1 \approx 0$.

\paragraph{Reverse process.} Generation proceeds by reversing the forward corruption: starting from a fully masked sequence $\mathbf{x}_T$ at $t = 1$, the model iteratively predicts and unmasks tokens until reaching the clean state at $t = 0$. A neural network $f_\theta$ parameterizes the reverse transition by producing a distribution over the vocabulary at each position:
\begin{equation}
    p_\theta(\mathbf{x}_{s} \mid \mathbf{x}_t) = \prod_{i=1}^{L} p_\theta(x_{s}^i \mid \mathbf{x}_t), \quad 0 \leq s < t \leq 1.
\end{equation}
For masked positions, the model predicts $p_\theta(x_s^i \mid \mathbf{x}_t) = \mathrm{Cat}(x_s^i; \pi_\theta^i(\mathbf{x}_t))$, where $\pi_\theta^i(\mathbf{x}_t) \in \Delta^{K}$ is the predicted distribution at position $i$. Unmasked tokens are preserved.

\subsection{Iterative Sampling}

Practical generation proceeds through $T$ discrete steps. At step $n$, let $\mathcal{M}_n = \{i : x_{t_n}^i = \texttt{[M]}\}$ denote the masked positions. The model computes predictions $\pi_\theta^i(\mathbf{x}_{t_n})$ for all $i \in \mathcal{M}_n$, then selects a subset $\mathcal{U}_n \subseteq \mathcal{M}_n$ to unmask based on a transfer schedule. A common strategy unmasks positions with the highest confidence $\max_{v} \pi_\theta^i(\mathbf{x}_{t_n})_v$, deferring uncertain positions to later steps.

\subsection{Training Objective}

DLMs minimize a variational upper bound on negative log-likelihood, which simplifies to a reweighted cross-entropy over masked positions:
\begin{equation}
    \mathcal{L}(\theta) = -\mathbb{E}_{\mathbf{x}_0, t, \mathbf{x}_t} \left[ w(t) \sum_{i=1}^{L} \mathbf{1}[x_t^i = \texttt{[M]}] \log p_\theta(x_0^i \mid \mathbf{x}_t) \right],
    \label{eq:training_loss}
\end{equation}
where $t \sim \mathcal{U}(0, 1)$ and $w(t)$ is a time-dependent weight, commonly $w(t) = 1/t$.

\section{Approach}
\label{sec:approach}

We present \JOT (\textbf{J}ust \textbf{o}n \textbf{T}ime), a training-free method for per-token early stopping in diffusion language models. The core idea is to finalize individual token predictions as soon as they exhibit sufficient confidence, rather than waiting for a fixed number of diffusion steps. Our approach adapts the acceptance threshold at each position based on spatial proximity to already-resolved tokens.

\subsection{Overview}
Standard DLM decoding runs for a predetermined number of steps $T$, unmasking tokens according to a transfer schedule. However, prior work has shown that predictions often stabilize well before the final diffusion step \citep{li2025diffusionlanguagemodelsknow}. \JOT exploits this observation by introducing an adaptive early-exit mechanism that operates at the token level: at each step, positions whose predictions exceed a dynamically computed confidence threshold are finalized immediately, removing them from further refinement.

The key components of our method are: (i) a \emph{confidence metric} that quantifies prediction certainty at each masked position; (ii) a \emph{spatial modulation} that lowers the threshold for positions adjacent to already-unmasked tokens.

\paragraph{Confidence Metric.} We measure prediction confidence using the ratio between the top two predicted probabilities. At each masked position $i$, the model produces logits $\ell^i \in \mathbb{R}^K$, which we convert to a probability distribution via softmax without temperature scaling: $\pi_\theta^i(\mathbf{x}_t) = \mathrm{softmax}(\ell^i) \in \Delta^K$. By using unscaled logits, our confidence metric remains invariant to any temperature parameter applied during sampling, decoupling early-exit decisions from generation diversity settings. Denote the largest and second-largest probabilities as $p_1^i$ and $p_2^i$, respectively. The confidence score is defined as:
\begin{equation}
    r^i = \frac{p_1^i}{p_2^i + \epsilon},
    \label{eq:confidence_ratio}
\end{equation}
where $\epsilon > 0$ is a small constant for numerical stability. This ratio captures how decisively the model favors its top prediction: $r^i \approx 1$ indicates uncertainty between alternatives, while $r^i \gg 1$ signals strong commitment to the leading candidate. Equivalently (neglecting $\epsilon$), $r^i$ is a \emph{logit margin}: $p_1^i/p_2^i = \exp(\ell^i_{(1)} - \ell^i_{(2)})$, where $\ell^i_{(1)}, \ell^i_{(2)}$ are the two largest logits, so the test $r^i \ge \tau$ asks for a top-2 logit gap of at least $\ln\tau$---about $4.5$ for $\tau = 90$ and $3.4$ for $\tau = 30$. This makes the temperature invariance concrete: a logit-margin threshold calibrated once transfers across sampling temperatures, whereas a probability threshold must be recalibrated per temperature---Fast-dLLM's threshold, calibrated for greedy decoding, drops from $53.7$ to $0.6$ pass@1 on Dream HumanEval at $T=0.1$ in our measurements. We claim transferability of the threshold, not identical decoding behavior across temperatures.

\paragraph{Spatial Modulation.} Recent work on diffusion language model training has shown that positions with stronger local context---particularly those near already-decoded tokens---tend to experience lower noise level \citep{ye2025dream7bdiffusionlarge}. Dream's CART weighting exploits this during training by applying context-adaptive noise rescheduling at the token level. We draw inspiration from this finding and apply analogous reasoning to inference: positions adjacent to already-unmasked tokens benefit from richer local context, making their predictions more reliable even at earlier steps. We capture this through a spatial softening factor based on proximity to resolved positions.

Let $\mathcal{M}_n$ denote the set of masked positions at step $n$. The spatial weight at position $i$ is computed using a geometric kernel over a window of radius $D$:
\begin{equation}
    w^i = \sum_{\substack{j \notin \mathcal{M}_n \\ |i-j| \leq D}} \gamma^{|i - j|}
    \label{eq:spatial_weight}
\end{equation}
where $\gamma \in (0, 1)$ is the decay rate controlling how quickly influence diminishes with distance. The spatial softening factor is the normalized weight:
\begin{equation}
    \phi^i = \min\left(1, \frac{w^i}{w_{\max}}\right),
    \label{eq:spatial_softening}
\end{equation}
where $w_{\max} = 2 \sum_{d=1}^{D} \gamma^d$ is the maximum possible weight. Positions at the boundary of masked regions receive higher $\phi^i$, reflecting their enhanced contextual support.

\paragraph{Adaptive Threshold.} The spatial softening factor modulates a position-specific acceptance threshold. The threshold at position $i$ is interpolated between a maximum $\tau_{\max}$ and minimum $\tau_{\min}$:
\begin{equation}
    \tau^i = \tau_{\max} - (\tau_{\max} - \tau_{\min}) \cdot \phi^i.
    \label{eq:threshold}
\end{equation}
Positions near unmasked tokens (high $\phi^i$) receive lower thresholds, making them easier to finalize. Positions in the interior of masked regions retain the stricter threshold $\tau_{\max}$, requiring higher confidence before commitment.

\paragraph{Early-Exit Decision.} At each step $n$, a masked position $i \in \mathcal{M}_n$ is finalized if its confidence exceeds the adaptive threshold:
\begin{equation}
    r^i \geq \tau^i.
    \label{eq:decision}
\end{equation}
When this condition is met, position $i$ is finalized and removed from the set of masked positions. The finalized token can be selected via argmax, $\hat{x}^i = \arg\max_v \pi_\theta^i(\mathbf{x}_{t_n})_v$, or sampled from the predicted distribution $\pi_\theta^i(\mathbf{x}_{t_n})$; we use argmax in our experiments but the method is compatible with sampling-based decoding. Positions that do not meet the threshold continue through the standard transfer schedule.

The complete procedure is summarized in \cref{alg:jot}.

\begin{algorithm}[tb]
    \caption{\JOT: Token-Level Early Stopping}
    \label{alg:jot}
    \begin{algorithmic}
        \State {\bfseries Input:} model $f_\theta$; prompt $\mathbf{x}_{\mathrm{prompt}}$; generation length $L$; steps $T$
        \State {\bfseries Parameters:} $\tau_{\max}, \tau_{\min}, \gamma, D$
        \State Initialize $\mathbf{x} \leftarrow [\mathbf{x}_{\mathrm{prompt}}; \texttt{[M]} \times L]$
        \For{$n = 1$ {\bfseries to} $T$}
            \State $\mathcal{M}_n \leftarrow \{i : x^i = \texttt{[M]}\}$
            \If{$\mathcal{M}_n = \emptyset$}
                \State {\bfseries break}
            \EndIf
            \State Compute logits $\ell \leftarrow f_\theta(\mathbf{x})$
            \State Compute probabilities $\pi^i \leftarrow \mathrm{softmax}(\ell^i)$ for $i \in \mathcal{M}_n$
            \State Compute confidence $r^i \leftarrow p_1^i / (p_2^i + \epsilon)$
            \For{$i \in \mathcal{M}_n$}
                \State Compute $\phi^i$ via \cref{eq:spatial_softening}
                \State $\tau^i \leftarrow \tau_{\max} - (\tau_{\max} - \tau_{\min}) \cdot \phi^i$
                \If{$r^i \geq \tau^i$}
                    \State $x^i \leftarrow \arg\max_v \pi^i_v$ \hfill \Comment{Finalize position}
                \EndIf
            \EndFor
            \State Apply transfer schedule to remaining masked positions
        \EndFor
        \State {\bfseries return} $\mathbf{x}$
    \end{algorithmic}
\end{algorithm}

\subsection{Discussion}

\paragraph{Computational overhead.} The spatial weight computation in \cref{eq:spatial_weight} can be implemented efficiently via a 1D convolution with a precomputed kernel, adding only a small per-step overhead (quantified by the wallclock analysis in \cref{app:wallclock}). The dominant cost remains the forward pass through $f_\theta$, which \JOT reduces by terminating early when all positions are finalized.

\paragraph{Interaction with transfer schedules.} \JOT operates alongside the standard transfer schedule rather than replacing it. At each step, positions that do not meet the early-exit criterion remain masked and are subject to the model's default unmasking schedule: the transfer schedule determines how many of these remaining positions to reveal based on model predictions. If no positions satisfy the confidence threshold, the step proceeds exactly as in standard decoding. This ensures that the method gracefully degrades to baseline behavior when predictions remain uncertain.

\paragraph{Implementation details.} We use $\epsilon = 10^{-12}$ in \cref{eq:confidence_ratio} and batch size $1$ throughout. The spatial kernel treats prompt tokens as unmasked: the convolution runs over the full canvas, so generation-initial positions adjacent to the prompt receive softened thresholds from the start. On LLaDA, \JOT operates within the current semi-autoregressive block only. The early-exit mask is applied on top of the transfer schedule (a logical OR) and is uncapped: if every remaining position clears its threshold, \JOT finalizes all of them in a single step.

\paragraph{Orthogonality with other acceleration methods.} \JOT reduces the \emph{number} of steps, whereas KV-caching methods such as dKV-Cache \citep{ma2025dkvcachecachediffusionlanguage} and D2F \citep{wang2025diffusionllmsfasterthanarinference} reduce the cost of \emph{each} step, so the two compose; we measure one such integration (\JOT inside Fast-dLLM's DualCache) in \cref{sec:experiments}.
\section{Why Early Exit Preserves Quality}
\label{sec:theory}

\JOT was motivated empirically. Here we give a short analysis that answers two questions: how much quality can be lost by committing a token once its confidence ratio exceeds $\tau$, and why a small gap remains even at very high thresholds.

\paragraph{Setup.}
Fix a masked position $i$ and, for the analysis, consider a reference continuation in which position $i$ remains masked while the other tokens are progressively revealed. Let $\mathcal{F}_n$ denote everything revealed by step $n$ (the prompt and all unmasked tokens). Write
    \begin{equation}
        M_n(v) \;=\; \mathbb{P}\!\left(x^i = v \mid \mathcal{F}_n\right)
        \label{eq:posterior}
    \end{equation}
for the true probability that position $i$ equals $v$ given the current context. As more tokens are revealed, $M_n(v)$ is a martingale: its current value equals the expectation of its future values, because a correctly updated probability has no predictable drift. This single property drives the analysis. We make one assumption: at the moment of commitment, the model's prediction matches the true conditional, $\pi^i_\theta = \mathbb{P}(x^i = \cdot \mid \mathcal{F}_n)$. This is exactly what the training objective in \cref{eq:training_loss} optimizes for---masked diffusion models are trained to predict the conditional distribution of clean tokens given an arbitrary partially masked context \citep{ou2026absorbingdiscretediffusionsecretly, zheng2025maskeddiffusionmodelssecretly}. We return to what happens when the assumption fails below.

\paragraph{Early exit and generation order.}
By the chain rule, a joint distribution factorizes identically under every generation order. As a consequence, a sequential decoder that commits one token at a time by sampling from the corresponding correct conditional and updates the context after each commitment produces the correct output distribution \emph{no matter which token it chooses to commit next}---even if the choice depends on the model's own confidences (\cref{lem:order} in \cref{app:proofs}).

\JOT instead fixes the current $\arg\max$, possibly at several positions in the same step. A central position-wise source of instability is that later context might have caused the current runner-up to overtake the selected token. The next result bounds how often that happens.

\begin{proposition}[How often can a committed token flip?]
    \label{prop:flip}
    Suppose position $i$ is committed at step $n$ with confidence ratio $r^i \ge \tau$, with top prediction $\hat v$ and runner-up $u$. Under the calibration assumption, and conditioned on $x^i$ lying in the leading pair $\{\hat v, u\}$,
        \begin{equation*}
            \mathbb{P}\!\left(\exists\, m \ge n:\; M_m(u) \ge M_m(\hat v) \,\middle|\, \mathcal{F}_n\right)
\;\le\; \frac{2}{1+\tau}.
        \end{equation*}
    Hence, under the same leading-pair condition, summing over at most $L$ committed positions gives a leading-pair reversal budget of at most $2L/(1+\tau)$.
\end{proposition}
The proof (\cref{app:proofs}) is short. At commit time the runner-up holds at most a $\frac{1}{1+\tau}$ share of the head-to-head probability against the leader. Because correctly updated probabilities form a martingale---a fair game---the chance that this share ever climbs to the $\tfrac{1}{2}$ required for a reversal is at most twice its current value (Ville's inequality; \citealp{ville1939etude}): the same reason a gambler holding \$1 in a fair game reaches \$45 with probability at most $1/45$. Spatial modulation is covered by the same leading-pair bound with the position-specific threshold: a token committed at $\tau^i$ contributes at most $2/(1+\tau^i)$.

\paragraph{The bound is consistent with the observed degradation.}
For GSM8K on Dream-7B ($L = 256$, $\tau_{\max} = 90$), \cref{prop:flip} gives a leading-pair reversal budget of at most $2 \cdot 256 / 91 \approx 5.6$ positions per sample---an overestimate, since most tokens commit with $r^i \gg \tau$. This is consistent with our measurements: on $500$ GSM8K samples, $87.4\%$ of outputs agree with full decoding on the final answer, and accuracy drops by $2.2$ points, a difference a McNemar exact test does not find significant ($p \approx 0.21$; \cref{app:failure_modes}).

The same scaling is also consistent with two side observations.
For short generations ($L = 3$ for MMLU, $L = 5$ for HellaSwag) the leading-pair reversal budget $2L/(1+\tau)$ is negligible at every threshold, matching the flat scores across $\tau$ in \cref{tab:threshold_ablation}; and LLaDA's block decoding caps the number of simultaneously masked positions at the block size, so that the same budget can be applied within each block, consistent with the empirical observation that a lower $\tau_{\max} = 30$ suffices for it.

\paragraph{Why quality does not fully recover at large $\tau$.}
\Cref{prop:flip} bounds the error of stopping too early under correct beliefs, and that error vanishes as $\tau$ grows. The residual gap at $\tau = 120$--$150$ (\cref{tab:threshold_ablation}) may therefore reflect failures of the calibration assumption itself. Splitting on whether the model's confidence can be trusted at commit time gives the following.

\begin{corollary}[Stopping error vs.\ calibration error]
    \label{cor:floor}
    Drop the calibration assumption, and let $q^*_n$ denote the true head-to-head share of the runner-up at commit time. Then for every $\kappa > 0$,
        \begin{equation}
            \mathbb{P}(\mathrm{flip}) \;\le\; \underbrace{\frac{2}{1+\kappa}}_{\text{stopping error}} \;+\; \underbrace{\mathbb{P}\!\left(r^i \ge \tau \;\text{ and }\; q^*_n > \tfrac{1}{1+\kappa}\right)}_{\text{calibration error }\, \varepsilon_{\mathrm{cal}}(\tau,\kappa)}.
            \label{eq:floor}
        \end{equation}
    If the model is calibrated, taking $\kappa = \tau$ recovers \cref{prop:flip}. In general, this decomposition does not imply that $\varepsilon_{\mathrm{cal}}$ decreases at the same $O(1/\tau)$ rate as the stopping term.
\end{corollary}

\Cref{cor:floor} clarifies the role of the confidence threshold. The stopping error shrinks at rate $O(1/\tau)$, while the calibration term is not controlled by the same inverse-threshold bound and may therefore limit the improvement obtained by increasing $\tau$.

This is consistent with what the failure analysis in \cref{app:failure_modes} observes: the errors that remain at high thresholds are single tokens committed with high confidence at reasoning branch points (e.g., ``$\times\,2$'' where ``$/\,2$'' is correct), and every such error type also occurs under full decoding.

In short: under the assumptions of \cref{prop:flip}, leading-pair reversals are controlled at rate $O(1/\tau)$, while residual differences may reflect calibration error that is not controlled by the same bound.

\section{Experiments}
\label{sec:experiments}

\begin{table*}[t]
    \centering
    \small
    \setlength{\tabcolsep}{1pt}  
    \begin{sc}
    \begin{tabular}{ll cc cc cc cc}
        \toprule
        & & \multicolumn{2}{c}{\textbf{GSM8K}} & \multicolumn{2}{c}{\textbf{MMLU}} & \multicolumn{2}{c}{\textbf{HellaSwag}} & \multicolumn{2}{c}{\textbf{HumanEval}} \\
        \cmidrule(lr){3-4} \cmidrule(lr){5-6} \cmidrule(lr){7-8} \cmidrule(lr){9-10}
        \textbf{Model} & \textbf{Method} & Score$\uparrow$ & Speed$\uparrow$ & Score$\uparrow$ & Speed$\uparrow$ & Score$\uparrow$ & Speed$\uparrow$ & Pass@1$\uparrow$ & Speed$\uparrow$ \\
        \midrule
        \multirow{5}{*}{Dream-7B} & Full Decoding & \underline{81.1} & 1.00$\times$ & \underline{68.2} & 1.00$\times$ & \underline{73.3} & 1.00$\times$ & \underline{59.1} & 1.00$\times$ \\
        & Prophet & 68.6 & 2.79$\times$ & 60.2 & 1.40$\times$ & 72.4 & \underline{1.94}$\times$ & 56.7 & \underline{9.17}$\times$ \\
        & KLASS & \textbf{82.7} & 2.45$\times$ & 64.4 & 1.06$\times$ & \textbf{74.8} & 1.13$\times$ & \textbf{59.8} & 7.92$\times$ \\
        & Fast-dLLM (DualCache) & 78.4 & \underline{4.08}$\times$ & \textbf{72.2} & \underline{1.49}$\times$ & 62.5 & 1.50$\times$ & 53.7 & 7.32$\times$ \\
        \rowcolor[gray]{0.92}
        & \JOT (Ours) & 78.8 & \textbf{5.54}$\times$ & 66.7 & \textbf{1.57}$\times$ & 72.7 & \textbf{2.26}$\times$ & 58.5 & \textbf{19.60}$\times$ \\
        \midrule
        \multirow{5}{*}{LLaDA-8B} & Full Decoding & \underline{74.5} & 1.00$\times$ & \textbf{67.3} & 1.00$\times$ & \underline{76.7} & 1.00$\times$ & \textbf{47.6} & 1.00$\times$ \\
        & Prophet & 64.4 & 2.74$\times$ & 63.3 & \textbf{2.08}$\times$ & 75.8 & 1.84$\times$ & 40.9 & 1.89$\times$ \\
        & KLASS & 74.2 & 2.58$\times$ & 63.4 & 1.26$\times$ & \textbf{77.1} & 1.48$\times$ & 40.2 & \underline{2.53}$\times$ \\
        & Fast-dLLM (DualCache) & \textbf{75.1} & \underline{3.40}$\times$ & \underline{65.2} & \underline{2.02}$\times$ & 74.2 & \textbf{3.83}$\times$ & 39.0 & \textbf{3.05}$\times$ \\
        \rowcolor[gray]{0.92}
        & \JOT (Ours) & 73.4 & \textbf{3.75}$\times$ & 64.5 & 1.98$\times$ & 76.6 & \underline{3.42}$\times$ & \underline{44.5} & 2.12$\times$ \\
        \bottomrule
    \end{tabular}
    \end{sc}
    \caption{Comparison of training-free acceleration methods on Dream-7B-Instruct and LLaDA-8B-Instruct across four benchmarks. We report task score (\%) with gain relative to full decoding, and speedup ($\times$). \JOT uses $\tau_{\max}=90$, $\tau_{\min}=1$, $\gamma=0.5$, $D=8$ for Dream and $\tau_{\max}=30$, $\tau_{\min}=1$, $\gamma=0.5$, $D=8$ for LLaDA. Fast-dLLM decodes greedily per its published protocol---its probability threshold is calibrated for greedy decoding and collapses under temperature sampling---whereas our Dream experiments sample at $T=0.1$. Binomial standard errors: ${\pm}1.1$ points on GSM8K ($n{=}1319$), ${\pm}3.8$--$3.9$ points on HumanEval ($n{=}164$). \textbf{Bold} and \underline{underline} mark best and second-best per column.}
    \label{tab:main_results}
\end{table*}

We evaluate \JOT on a diverse set of benchmarks spanning reasoning, knowledge, and code generation. Our experiments aim to answer: (1) Does \JOT achieve better speed-quality trade-offs than existing early-stopping methods? (2) How do threshold selection and spatial modulation contribute to performance?

\subsection{Experimental Setup}

We evaluate on two state-of-the-art open diffusion language models: Dream-7B-Instruct \citep{ye2025dream7bdiffusionlarge} and LLaDA-8B-Instruct \citep{nie2025largelanguagediffusionmodels}. We consider four benchmarks: GSM8K \citep{cobbe2021gsm8k} for mathematical reasoning with chain-of-thought prompting, MMLU \citep{hendrycks2021measuringmassivemultitasklanguage} for multitask language understanding, HellaSwag \citep{zellers2019hellaswagmachinereallyfinish} for commonsense reasoning, and HumanEval \citep{chen2021codex} for code generation. All experiments use zero-shot prompting. Configuration details are provided in \cref{app:config}.

We compare against full decoding (standard diffusion sampling), Prophet \citep{li2025diffusionlanguagemodelsknow} (early-commit with top-2 confidence gaps), KLASS \citep{kim2025klassklguidedfastinference} (KL-adaptive stability sampling), and Fast-dLLM \citep{wu2025fastdllmtrainingfreeaccelerationdiffusion} (confidence-gated parallel unmasking with block-wise KV caching). All baselines use their recommended configurations: Prophet uses $\tau_{\text{high}}=7.5$, $\tau_{\text{mid}}=5.0$, $\tau_{\text{low}}=2.5$ for both models; KLASS uses $\tau_{\text{KL}}=0.001$, $\tau_{\text{conf}}=0.9$ for Dream and $\tau_{\text{KL}}=0.01$, $\tau_{\text{conf}}=0.7$ for LLaDA; Fast-dLLM uses its DualCache variant with confidence threshold $0.9$ and greedy decoding, following its published protocol. We report task-specific scores and \emph{speedup}, defined as the ratio of configured to actual steps. All experiments use LM-Evaluation-Harness \citep{eval-harness} for reproducibility. Our implementation is built upon the open-source framework provided by \citep{dllm}, which we extend to incorporate \JOT.

Ablation studies (\cref{sec:threshold_ablation,sec:spatial_ablation,sec:sweep}) are conducted on Dream-7B; a separate LLaDA threshold ablation is provided in \cref{app:llada_ablation}.

\subsection{Main Results}

\Cref{tab:main_results} presents our main comparison. On Dream-7B, \JOT achieves substantial speedups---$5.54\times$ on GSM8K and $19.60\times$ on HumanEval---while incurring only modest accuracy drops ($-2.3$ and $-0.6$ points respectively). Prophet shows larger quality degradation, particularly on GSM8K where it drops $12.5$ points, and on MMLU where it drops $8.0$ points. KLASS maintains strong accuracy on GSM8K ($+1.6$) but with lower speedup ($2.45\times$) and shows significant MMLU degradation ($-3.8$ points). Fast-dLLM is a strong baseline on both models: on Dream-7B, \JOT matches or exceeds it on both score and speedup for GSM8K, HellaSwag, and HumanEval, while Fast-dLLM scores higher on MMLU ($72.2$) but degrades sharply on HellaSwag ($62.5$ vs.\ $73.3$ for full decoding); on LLaDA-8B, Fast-dLLM attains the best GSM8K score ($75.1$, above full decoding) and the highest HellaSwag and HumanEval speedups, while \JOT scores higher on HellaSwag ($76.6$ vs.\ $74.2$) and HumanEval ($44.5$ vs.\ $39.0$). We note that, given the binomial standard error of ${\pm}3.8$--$3.9$ points at $n=164$, all HumanEval score differences between methods lie within one standard error. On LLaDA-8B (threshold selection detailed in \cref{app:llada_ablation}), \JOT maintains accuracy within $3.1$ points of full decoding on all benchmarks while providing competitive speedups. The per-token approach allows positions that genuinely require further refinement to continue through the diffusion chain, while finalizing confident predictions early. Wallclock time analysis (\cref{app:wallclock}) confirms that computational overhead is minor, generation quality metrics (\cref{app:quality}) show that fluency is preserved, a failure mode analysis (\cref{app:failure_modes}) characterizes the remaining error patterns, and confidence dynamics traces (\cref{app:dynamics}) illustrate how \JOT reallocates decoding effort away from ``obvious'' context tokens and toward the reasoning phase.

\Cref{fig:pareto} visualizes the aggregate speed-quality trade-off across all benchmarks on Dream-7B, using the geometric mean of step-speedups. \JOT occupies a favorable position on the Pareto frontier, achieving $4.43\times$ aggregate speedup while retaining $98.3\%$ of baseline quality. In contrast, Prophet sacrifices quality for speed ($91.9\%$ retention at $2.89\times$), KLASS preserves quality ($99.9\%$) but with more limited acceleration ($2.20\times$), and Fast-dLLM sits in between ($94.7\%$ retention at $2.86\times$). This positions \JOT as a compelling choice when both speed and accuracy matter.

\paragraph{Composability with KV-cache acceleration.}
\JOT decides \emph{when} tokens finalize, whereas Fast-dLLM's DualCache reduces the cost of \emph{each} step, so the two can be combined; we measure this combination directly. Running \JOT inside DualCache (with $\tau_{\max}$ retuned to $30$ for Dream and $60$ for LLaDA, since cached suffix logits are approximate), we obtain $76.5$ at $4.22\times$ on Dream GSM8K, $50.6$ at $8.47\times$ on Dream HumanEval, $75.7$ at $3.21\times$ on LLaDA GSM8K---the best LLaDA GSM8K score of any configuration we test, above full decoding's $74.5$---and $36.6$ at $3.02\times$ on LLaDA HumanEval. The retuned thresholds come from a small $\tau$ sweep under the cache (\cref{app:composability}). Composition thus preserves GSM8K accuracy, while on HumanEval Fast-dLLM's max-probability gate proves ${\sim}3$ points more robust to the approximate suffix logits than our confidence ratio---a gap that retuning narrows but does not close. Block-wise caching also caps the exit horizon---exits cannot cross block boundaries---reducing \JOT's standalone $19.60\times$ on Dream HumanEval to $8.47\times$ under the cache.

\begin{table*}[t]
    \centering
    \small
    \setlength{\tabcolsep}{4pt}
    \begin{sc}
    \begin{tabular}{c cc cc cc cc}
        \toprule
        & \multicolumn{2}{c}{\textbf{GSM8K}} & \multicolumn{2}{c}{\textbf{MMLU}} & \multicolumn{2}{c}{\textbf{HellaSwag}} & \multicolumn{2}{c}{\textbf{HumanEval}} \\
        \cmidrule(lr){2-3} \cmidrule(lr){4-5} \cmidrule(lr){6-7} \cmidrule(lr){8-9}
        $\tau$ & Score$\uparrow$ & Speed$\uparrow$ & Score$\uparrow$ & Speed$\uparrow$ & Score$\uparrow$ & Speed$\uparrow$ & Pass@1$\uparrow$ & Speed$\uparrow$ \\
        \midrule
        \textit{Full Decoding} & \textbf{81.1} & 1.00$\times$ & \textbf{68.2} & 1.00$\times$ & \textbf{73.3} & 1.00$\times$ & \textbf{59.1} & 1.00$\times$ \\
        \midrule
        10 & 61.9 & \textbf{11.07}$\times$ & 66.7 & \textbf{2.42}$\times$ & 72.6 & \textbf{2.40}$\times$ & 43.3 & \textbf{42.60}$\times$ \\
        30 & 74.6 & \underline{7.12}$\times$ & 66.7 & \underline{1.86}$\times$ & 72.6 & \underline{2.31}$\times$ & 54.3 & \underline{28.12}$\times$ \\
        60 & 76.1 & 5.90$\times$ & 66.6 & 1.60$\times$ & 72.6 & 2.27$\times$ & \underline{57.3} & 21.53$\times$ \\
        90 & 78.8 & 5.30$\times$ & 66.7 & 1.51$\times$ & \underline{72.8} & 2.25$\times$ & 56.7 & 19.08$\times$ \\
        120 & \underline{80.7} & 4.88$\times$ & 66.7 & 1.48$\times$ & 72.6 & 2.23$\times$ & \textbf{59.1} & 17.28$\times$ \\
        150 & \underline{80.7} & 4.67$\times$ & \underline{66.8} & 1.46$\times$ & 72.7 & 2.21$\times$ & 56.7 & 16.46$\times$ \\
        \bottomrule
    \end{tabular}
    \end{sc}
    \caption{Ablation on confidence threshold without spatial modulation. We report score (\%) with gain relative to full decoding, and speedup ($\times$). Results on Dream-7B. \textbf{Bold} and \underline{underline} mark best and second-best per column.}
    \label{tab:threshold_ablation}
\end{table*}

\begin{table*}[t]
    \centering
    \small
    \setlength{\tabcolsep}{4pt}
    \begin{sc}
    \begin{tabular}{ccc cc cc cc cc}
        \toprule
        & & & \multicolumn{2}{c}{\textbf{GSM8K}} & \multicolumn{2}{c}{\textbf{MMLU}} & \multicolumn{2}{c}{\textbf{HellaSwag}} & \multicolumn{2}{c}{\textbf{HumanEval}} \\
        \cmidrule(lr){4-5} \cmidrule(lr){6-7} \cmidrule(lr){8-9} \cmidrule(lr){10-11}
        $\tau_{\max}$ & $\gamma$ & $D$ & Score$\uparrow$ & Speed$\uparrow$ & Score$\uparrow$ & Speed$\uparrow$ & Score$\uparrow$ & Speed$\uparrow$ & Pass@1$\uparrow$ & Speed$\uparrow$ \\
        \midrule
        \multicolumn{3}{c}{\textit{Full Decoding}} & \textbf{81.1} & 1.00$\times$ & \textbf{68.2} & 1.00$\times$ & \textbf{73.3} & 1.00$\times$ & \textbf{59.1} & 1.00$\times$ \\
        \midrule
        60 & 0.3 & 8 & 75.5 & 6.11$\times$ & 66.7 & 1.66$\times$ & \underline{72.8} & \underline{2.27}$\times$ & 55.5 & \underline{22.39}$\times$ \\
        60 & 0.5 & 8 & 76.7 & \textbf{6.25}$\times$ & \underline{66.9} & \underline{1.68}$\times$ & 72.7 & \underline{2.27}$\times$ & 53.7 & 22.01$\times$ \\
        60 & 0.5 & 16 & 74.6 & \underline{6.15}$\times$ & \underline{66.9} & \underline{1.68}$\times$ & 72.7 & \underline{2.27}$\times$ & 54.9 & 22.38$\times$ \\
        60 & 0.7 & 8 & 74.9 & \textbf{6.25}$\times$ & 66.7 & \textbf{1.75}$\times$ & 72.6 & \textbf{2.28}$\times$ & 56.1 & \textbf{22.79}$\times$ \\
        \midrule
        90 & 0.3 & 8 & 78.4 & 5.52$\times$ & 66.7 & 1.54$\times$ & 72.7 & 2.26$\times$ & 56.1 & 19.80$\times$ \\
        90 & 0.5 & 8 & 78.8 & 5.54$\times$ & 66.7 & 1.57$\times$ & 72.7 & 2.26$\times$ & \underline{58.5} & 19.60$\times$ \\
        90 & 0.5 & 16 & 78.2 & 5.57$\times$ & 66.7 & 1.57$\times$ & 72.7 & 2.26$\times$ & 57.9 & 19.71$\times$ \\
        90 & 0.7 & 8 & 78.5 & 5.63$\times$ & 66.8 & 1.60$\times$ & 72.7 & 2.26$\times$ & 57.9 & 19.88$\times$ \\
        \midrule
        120 & 0.3 & 8 & 79.9 & 5.07$\times$ & 66.8 & 1.49$\times$ & 72.7 & 2.25$\times$ & 56.7 & 17.68$\times$ \\
        120 & 0.5 & 8 & 79.9 & 5.12$\times$ & 66.7 & 1.51$\times$ & 72.7 & 2.25$\times$ & 56.7 & 17.90$\times$ \\
        120 & 0.5 & 16 & 79.8 & 5.12$\times$ & 66.7 & 1.51$\times$ & 72.7 & 2.25$\times$ & 57.3 & 17.85$\times$ \\
        120 & 0.7 & 8 & \underline{80.0} & 5.18$\times$ & 66.7 & 1.54$\times$ & 72.7 & 2.25$\times$ & \underline{58.5} & 18.15$\times$ \\
        \bottomrule
    \end{tabular}
    \end{sc}
    \caption{Hyperparameter sweep on Dream-7B-Instruct. Score (\%) with gain relative to full decoding, and speedup ($\times$). \textbf{Bold} and \underline{underline} mark best and second-best per column.}
    \label{tab:sweep}
\end{table*}

\subsection{Threshold Ablation}
\label{sec:threshold_ablation}

We analyze the effect of the confidence threshold $\tau$ without spatial modulation ($\tau^i = \tau$ for all positions). This isolates the impact of threshold selection on the speed-quality trade-off.

\Cref{tab:threshold_ablation} reveals that threshold selection critically affects both quality and speedup. Lower thresholds ($\tau \leq 30$) yield aggressive speedups but cause substantial accuracy drops---GSM8K degrades by up to $19.2$ points at $\tau=10$ and HumanEval by $15.8$ points. The threshold $\tau = 90$ provides a favorable balance, achieving $5.30\times$ speedup on GSM8K with only $-2.3$ points accuracy drop and $19.08\times$ on HumanEval with $-2.4$ points. Higher thresholds ($\tau = 120, 150$) approach baseline quality more closely (GSM8K within $0.4$ points) though with diminishing speedup returns. Based on these results, we select $\tau_{\max} = 90$ as the baseline threshold for spatial modulation experiments.

\subsection{Spatial Modulation Ablation}
\label{sec:spatial_ablation}

Having identified $\tau_{\max} = 90$ as an effective threshold, we sweep the decay rate $\gamma \in \{0.3, 0.5, 0.7, 0.9\}$ and window size $D \in \{4, 8, 16\}$ with $\tau_{\min} = 1$ fixed (full results in \cref{app:spatial_ablation}; an ablation on $\tau_{\min}$ itself is in \cref{app:tau_min}). Spatial modulation buys speedup on long-generation benchmarks: HumanEval improves from $19.08\times$ to $21.27\times$ at $(\gamma, D) = (0.9, 16)$, while larger decay rates spread influence too broadly and cost accuracy on GSM8K ($-1.5$ points at $\gamma = 0.9$, $D = 8$). We select $(\gamma, D) = (0.5, 8)$ as the default, which matches the no-spatial GSM8K accuracy ($78.8\%$) while improving speedup from $5.30\times$ to $5.54\times$.

\subsection{Hyperparameter Sweep}
\label{sec:sweep}

Finally, we conduct a sweep around the identified optimal parameters: $\tau_{\max} \in \{60, 90, 120\}$ and $(\gamma, D) \in \{(0.3, 8), (0.5, 8), (0.5, 16), (0.7, 8)\}$, with $\tau_{\min} = 1$ fixed.

\Cref{tab:sweep} confirms that \JOT maintains robust performance across a range of hyperparameter configurations. Lower thresholds ($\tau_{\max} = 60$) trade more accuracy for speed, while higher thresholds ($\tau_{\max} = 120$) approach baseline quality with more modest speedups. Inspecting mean confidence across diffusion steps (\cref{app:dynamics}) shows why this ordering holds: conservative settings ($\tau_{\max} \geq 30$) let the model finalize obvious context tokens immediately while preserving a reasoning phase whose confidence trajectory mirrors full decoding, whereas aggressive settings truncate that phase---the mechanism behind the degradation in \cref{tab:threshold_ablation}. The configuration $(\tau_{\max}, \gamma, D) = (90, 0.5, 8)$ provides a balanced default, but practitioners can adjust based on their accuracy-speed priorities. This stability reduces the need for task-specific tuning, making the method practical for deployment across diverse applications.
\begin{figure}[tb]
    \centering
    \includegraphics[width=0.7\columnwidth]{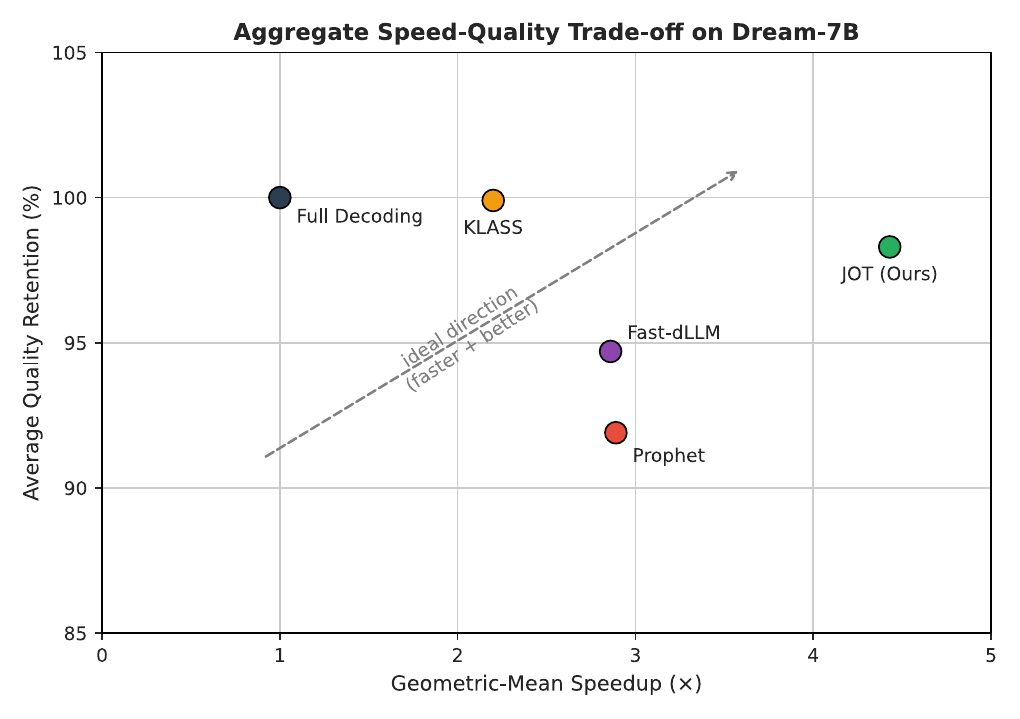}
    \caption{Aggregate speed-quality trade-off on Dream-7B. Speedup is the geometric mean of step-speedups across the four benchmarks; quality retention is the mean score ratio relative to full decoding. \JOT attains the highest aggregate speedup ($4.43\times$) while retaining $98.3\%$ of baseline quality.}
    \label{fig:pareto}
\end{figure}

\section{Limitations}
\label{sec:limitations}

\Cref{sec:theory} provides a formal position-wise stability
analysis for confidence-guided commitments. Under the assumption that
the model prediction matches the target conditional, it shows that
the probability of a leading-pair reversal decreases as
$O(1/\tau)$, while the subsequent decomposition separates this
threshold-controlled term from residual calibration error. Extending
the analysis to direct guarantees on sequence-level quality and
interactions among multiple tokens committed in the same step remains
an open direction.

Although the hyperparameter sweep shows reasonable robustness,
optimal thresholds may still vary across tasks and domains. Our
experiments focus on instruction-following benchmarks. Other domains
like creative writing or translation may require different
configurations. We also note that the thresholds in this work were
selected on the same evaluation sets used for reporting, without a
held-out calibration split; in deployment we recommend calibrating
$\tau$ on a small held-out subset.

Our experiments use generation lengths of at most 512 tokens
(\cref{app:config}). Behavior on longer sequences (e.g., 1024 or 2048
tokens) is unexplored. Spatial modulation effects may compound
differently at scale, and the cumulative impact of many early-exit
decisions on coherence is unknown.
\section{Conclusion}
\label{sec:conclusion}

We presented \JOT, a training-free method for per-token early stopping in diffusion language models. By monitoring prediction confidence at each position independently and applying spatially-adaptive thresholds, \JOT allows different tokens to exit at different steps, concentrating computation on positions that genuinely require further refinement. Experiments on Dream-7B and LLaDA-8B across diverse benchmarks demonstrate that \JOT achieves substantial speedups while maintaining competitive accuracy, providing the best speed--quality trade-off among the tested early-stopping methods on Dream-7B and remaining competitive on LLaDA-8B.

\bibliography{main}

\appendix
\section{Experimental Configuration}
\label{app:config}

\begin{table}[t]
    \centering
    \small
    \begin{sc}
    \begin{tabular}{llcc}
        \toprule
        \textbf{Model} & \textbf{Benchmark} & \textbf{Steps} & \textbf{Block} \\
        \midrule
        \multirow{4}{*}{Dream-7B} & GSM8K & 256 & --- \\
        & MMLU & 3 & --- \\
        & HellaSwag & 5 & --- \\
        & HumanEval & 512 & --- \\
        \midrule
        \multirow{4}{*}{LLaDA-8B} & GSM8K & 256 & 32 \\
        & MMLU & 3 & 3 \\
        & HellaSwag & 5 & 5 \\
        & HumanEval & 512 & 32 \\
        \bottomrule
    \end{tabular}
    \end{sc}
    \caption{Evaluation configuration for each model and benchmark.}
    \label{tab:config}
\end{table}

\textbf{Steps} denotes the number of diffusion (denoising) steps used during generation; \textbf{Block} is the block size for LLaDA's semi-autoregressive decoding, which generates tokens in fixed-size blocks (Dream uses fully parallel decoding, so no block size applies). These are standard configurations commonly adopted in the literature.

\section{Additional Experiments}
\label{app:additional}

\subsection{Wallclock Time Analysis}
\label{app:wallclock}

The step-based speedup reported in the main text measures the reduction in diffusion steps, which directly correlates with the number of forward passes through the model. However, practical deployment also involves per-step overhead from confidence computation and early-exit bookkeeping. \Cref{tab:wallclock} reports wallclock speedups measured on a single NVIDIA A100 GPU.

\begin{table*}[t]
    \centering
    \small
    \setlength{\tabcolsep}{4pt}
    \begin{sc}
    \begin{tabular}{ll cc cc cc cc}
        \toprule
        & & \multicolumn{2}{c}{\textbf{GSM8K}} & \multicolumn{2}{c}{\textbf{MMLU}} & \multicolumn{2}{c}{\textbf{HellaSwag}} & \multicolumn{2}{c}{\textbf{HumanEval}} \\
        \cmidrule(lr){3-4} \cmidrule(lr){5-6} \cmidrule(lr){7-8} \cmidrule(lr){9-10}
        \textbf{Model} & \textbf{Method} & Score$\uparrow$ & Speed$\uparrow$ & Score$\uparrow$ & Speed$\uparrow$ & Score$\uparrow$ & Speed$\uparrow$ & Pass@1$\uparrow$ & Speed$\uparrow$ \\
        \midrule
        \multirow{4}{*}{Dream-7B} & Full Decoding & \underline{81.1} & 1.00$\times$ & \textbf{68.2} & 1.00$\times$ & \underline{73.3} & 1.00$\times$ & \underline{59.1} & 1.00$\times$ \\
        & Prophet & 68.6 & \underline{2.41}$\times$ & 60.2 & \underline{1.38}$\times$ & 72.4 & \underline{1.91}$\times$ & 56.7 & \underline{6.82}$\times$ \\
        & KLASS & \textbf{82.7} & 2.12$\times$ & 64.4 & 1.08$\times$ & \textbf{74.8} & 1.05$\times$ & \textbf{59.8} & 5.94$\times$ \\
        \rowcolor[gray]{0.92}
        & \JOT (Ours) & 78.8 & \textbf{4.71}$\times$ & \underline{66.7} & \textbf{1.55}$\times$ & 72.7 & \textbf{2.22}$\times$ & 58.5 & \textbf{15.15}$\times$ \\
        \midrule
        \multirow{4}{*}{LLaDA-8B} & Full Decoding & \textbf{74.5} & 1.00$\times$ & \textbf{67.3} & 1.00$\times$ & \underline{76.7} & 1.00$\times$ & \textbf{47.6} & 1.00$\times$ \\
        & Prophet & 64.4 & \underline{2.38}$\times$ & 63.3 & \textbf{2.04}$\times$ & 75.8 & \underline{1.81}$\times$ & 40.9 & 1.72$\times$ \\
        & KLASS & \underline{74.2} & 1.92$\times$ & 63.4 & 1.24$\times$ & \textbf{77.1} & 1.41$\times$ & 40.2 & \textbf{2.18}$\times$ \\
        \rowcolor[gray]{0.92}
        & \JOT (Ours) & 73.4 & \textbf{3.15}$\times$ & \underline{64.5} & \underline{1.85}$\times$ & 76.6 & \textbf{3.08}$\times$ & \underline{44.5} & \underline{1.77}$\times$ \\
        \bottomrule
    \end{tabular}
    \end{sc}
    \caption{Wallclock time comparison on Dream-7B-Instruct and LLaDA-8B-Instruct. We report task score (\%) with gain relative to full decoding, and wallclock speedup ($\times$). \textbf{Bold} and \underline{underline} mark best and second-best per column.}
    \label{tab:wallclock}
\end{table*}

For long-generation tasks (GSM8K, HumanEval), the overhead is amortized over many steps, and wallclock speedups remain substantial: \JOT achieves $4.71\times$ on GSM8K and $15.15\times$ on HumanEval for Dream-7B. For short-generation tasks (MMLU, HellaSwag), the overhead represents a larger fraction of total time, resulting in wallclock speedups closer to step-based speedups. Overall, the computational overhead of \JOT is minor relative to the savings from reduced forward passes.

\subsection{Effect of Minimum Threshold}
\label{app:tau_min}

The spatial modulation mechanism interpolates thresholds between $\tau_{\max}$ and $\tau_{\min}$ based on proximity to unmasked tokens. We investigate whether the choice of $\tau_{\min}$ significantly affects performance by comparing $\tau_{\min} = 1$ (used in main experiments) with $\tau_{\min} = 30$.

\begin{table*}[t]
    \centering
    \small
    \begin{sc}
    \begin{tabular}{cc cccc cccc}
        \toprule
        & & \multicolumn{4}{c}{\textbf{GSM8K}} & \multicolumn{4}{c}{\textbf{HumanEval}} \\
        \cmidrule(lr){3-6} \cmidrule(lr){7-10}
        & & \multicolumn{2}{c}{$\tau_{\min}=1$} & \multicolumn{2}{c}{$\tau_{\min}=30$} & \multicolumn{2}{c}{$\tau_{\min}=1$} & \multicolumn{2}{c}{$\tau_{\min}=30$} \\
        \cmidrule(lr){3-4} \cmidrule(lr){5-6} \cmidrule(lr){7-8} \cmidrule(lr){9-10}
        $\gamma$ & $D$ & Score & Speed & Score & Speed & Pass@1 & Speed & Pass@1 & Speed \\
        \midrule
        0.9 & 16 & 78.5 & \textbf{5.82}$\times$ & \textbf{78.0} & \textbf{5.51}$\times$ & 57.3 & \textbf{21.94}$\times$ & \underline{58.0} & 19.63$\times$ \\
        0.9 & 8 & 77.9 & \underline{5.71}$\times$ & \textbf{78.0} & \underline{5.38}$\times$ & \underline{58.7} & \underline{20.82}$\times$ & 57.3 & \underline{19.78}$\times$ \\
        0.7 & 8 & \underline{79.7} & 5.58$\times$ & 77.3 & 5.34$\times$ & 57.3 & 20.22$\times$ & \textbf{58.7} & 19.71$\times$ \\
        0.5 & 16 & \underline{79.7} & 5.49$\times$ & 77.3 & 5.31$\times$ & \underline{58.7} & 20.03$\times$ & 57.3 & 19.73$\times$ \\
        0.5 & 8 & \textbf{80.0} & 5.49$\times$ & 77.3 & 5.30$\times$ & \underline{58.7} & 19.97$\times$ & \textbf{58.7} & 19.70$\times$ \\
        0.5 & 4 & 79.4 & 5.45$\times$ & \textbf{78.0} & 5.28$\times$ & 58.0 & 19.97$\times$ & \underline{58.0} & \textbf{19.85}$\times$ \\
        0.3 & 8 & \underline{79.7} & 5.45$\times$ & 77.3 & 5.27$\times$ & \textbf{59.3} & 19.90$\times$ & \textbf{58.7} & 19.61$\times$ \\
        \bottomrule
    \end{tabular}
    \end{sc}
    \caption{Comparison of $\tau_{\min} = 1$ vs.\ $\tau_{\min} = 30$ with $\tau_{\max} = 90$. Higher $\tau_{\min}$ provides no meaningful benefit. Results on Dream-7B (random subset of each benchmark). \textbf{Bold} and \underline{underline} mark best and second-best per column.}
    \label{tab:tau_min_ablation}
\end{table*}

\Cref{tab:tau_min_ablation} shows that increasing $\tau_{\min}$ from 1 to 30 provides no meaningful benefit. On GSM8K, accuracy remains within 1-2 points across configurations, while speedups are slightly lower with $\tau_{\min} = 30$ due to the reduced range of threshold modulation. On HumanEval, both settings achieve similar accuracy (57--59\%) with comparable speedups. These results suggest that the primary driver of performance is $\tau_{\max}$, and positions benefiting from spatial modulation (those near context boundaries) can safely use very low thresholds without degrading quality. We therefore use $\tau_{\min} = 1$ as the default setting.

\subsection{Generation Quality Metrics}
\label{app:quality}

To assess whether early stopping affects generation fluency beyond task accuracy, we evaluate open-ended text continuation on 128 samples from the C4 validation set \citep{raffel2023exploringlimitstransferlearning}. Using the first 128 tokens as a prefix, we force the generation of 256 continuation tokens using Dream-v0-Instruct-7B. Table \ref{tab:quality} reports MAUVE \citep{pillutla2021mauvemeasuringgapneural} for distributional similarity against the human reference, 4-gram repetition rate (Rep-4), and bigram diversity (Dist-2).

\begin{table}[t]
    \centering
    \small
    \begin{tabular}{lccc}
        \toprule
        \textsc{Method} & \textsc{Mauve}$\uparrow$ & \textsc{Rep-4}$\downarrow$ & \textsc{Dist-2}$\uparrow$ \\
        \midrule
        Human Reference & 1.0000 & 0.0080 & 0.7300 \\
        \midrule
        Full Decoding & 0.8975 & 0.0303 & 0.7479 \\
        \JOT (argmax) & \textbf{0.9251} & \textbf{0.0286} & \textbf{0.7486} \\
        \bottomrule
    \end{tabular}
    \caption{Generation quality metrics on C4 continuations (256 tokens). Lower is better for Rep-4 and Steps; higher is better for MAUVE and Dist-2. \textbf{Bold} marks the best model performance.}
    \label{tab:quality}
\end{table}

The baseline decodes with nucleus sampling ($p=0.95$, $T=1.0$) with EOS logits masked to force full-length outputs, while \JOT is configured with argmax filling and a maximum threshold of $90$. The two runs therefore differ in decoding policy in addition to early exit, so these numbers should not be read as a controlled superiority claim. \JOT attains a higher MAUVE score (0.9251 vs.\ 0.8975), slightly lower repetition (0.0286 vs.\ 0.0303), and marginally higher diversity (0.7486 vs.\ 0.7479), while reducing the average number of generation steps. We conclude that early exiting introduces no measurable degradation in open-ended generation.

\subsection{LLaDA Threshold Ablation}
\label{app:llada_ablation}

\begin{table*}[t]
    \centering
    \small
    \begin{sc}
    \setlength{\tabcolsep}{4pt}
    \begin{tabular}{l cc cc cc cc}
        \toprule
        & \multicolumn{2}{c}{\textbf{GSM8K}} & \multicolumn{2}{c}{\textbf{MMLU}} & \multicolumn{2}{c}{\textbf{HellaSwag}} & \multicolumn{2}{c}{\textbf{HumanEval}} \\
        \cmidrule(lr){2-3} \cmidrule(lr){4-5} \cmidrule(lr){6-7} \cmidrule(lr){8-9}
        $\tau_{\max}$ & Score$\uparrow$ & Speed$\uparrow$ & Score$\uparrow$ & Speed$\uparrow$ & Score$\uparrow$ & Speed$\uparrow$ & Pass@1$\uparrow$ & Speed$\uparrow$ \\
        \midrule
        Baseline & 74.5 & 1.00$\times$ & \textbf{67.3} & 1.00$\times$ & \textbf{76.7} & 1.00$\times$ & \textbf{47.6} & 1.00$\times$ \\
        \midrule
        10 & 32.2 & \textbf{8.97}$\times$ & 63.8 & \textbf{2.21}$\times$ & 72.7 & \textbf{4.16}$\times$ & 41.4 & \textbf{3.50}$\times$ \\
        30 & 72.1 & \underline{3.54}$\times$ & 63.8 & \underline{1.98}$\times$ & 72.7 & \underline{3.27}$\times$ & 45.1 & \underline{2.12}$\times$ \\
        60 & 75.8 & 2.69$\times$ & 63.8 & 1.85$\times$ & 72.7 & 2.03$\times$ & \underline{45.5} & 1.77$\times$ \\
        90 & \textbf{79.8} & 2.45$\times$ & 63.8 & 1.77$\times$ & 72.7 & 1.64$\times$ & 44.5 & 1.62$\times$ \\
        120 & \underline{79.1} & 2.36$\times$ & 63.8 & 1.71$\times$ & 72.7 & 1.33$\times$ & 44.1 & 1.51$\times$ \\
        150 & 78.2 & 2.24$\times$ & 63.8 & 1.65$\times$ & 72.7 & 1.21$\times$ & 43.9 & 1.44$\times$ \\
        \bottomrule
    \end{tabular}
    \end{sc}
    \caption{Threshold ablation on LLaDA-8B-Instruct (random subset of each benchmark). Spatial modulation uses $\gamma=0.5$, $D=8$. Score gain is relative to full decoding baseline. \textbf{Bold} and \underline{underline} mark best and second-best per column.}
    \label{tab:llada_ablation}
\end{table*}

To verify generalization across model architectures, we conduct a threshold ablation on LLaDA-8B (\cref{tab:llada_ablation}). LLaDA exhibits different sensitivity patterns than Dream: GSM8K benefits from higher thresholds ($\tau=60$--$90$ achieve accuracy above baseline), while MMLU and HellaSwag show consistent scores across all thresholds with varying speedups. Based on these results, we select $\tau_{\max}=30$ for LLaDA in the main comparison, as it provides strong speedups ($3.54\times$ on GSM8K, $1.98\times$ on MMLU) while maintaining close to baseline quality on reasoning tasks.

\subsection{Confidence Dynamics}
\label{app:dynamics}

\begin{figure}[t]
    \centering
    \includegraphics[width=0.8\columnwidth]{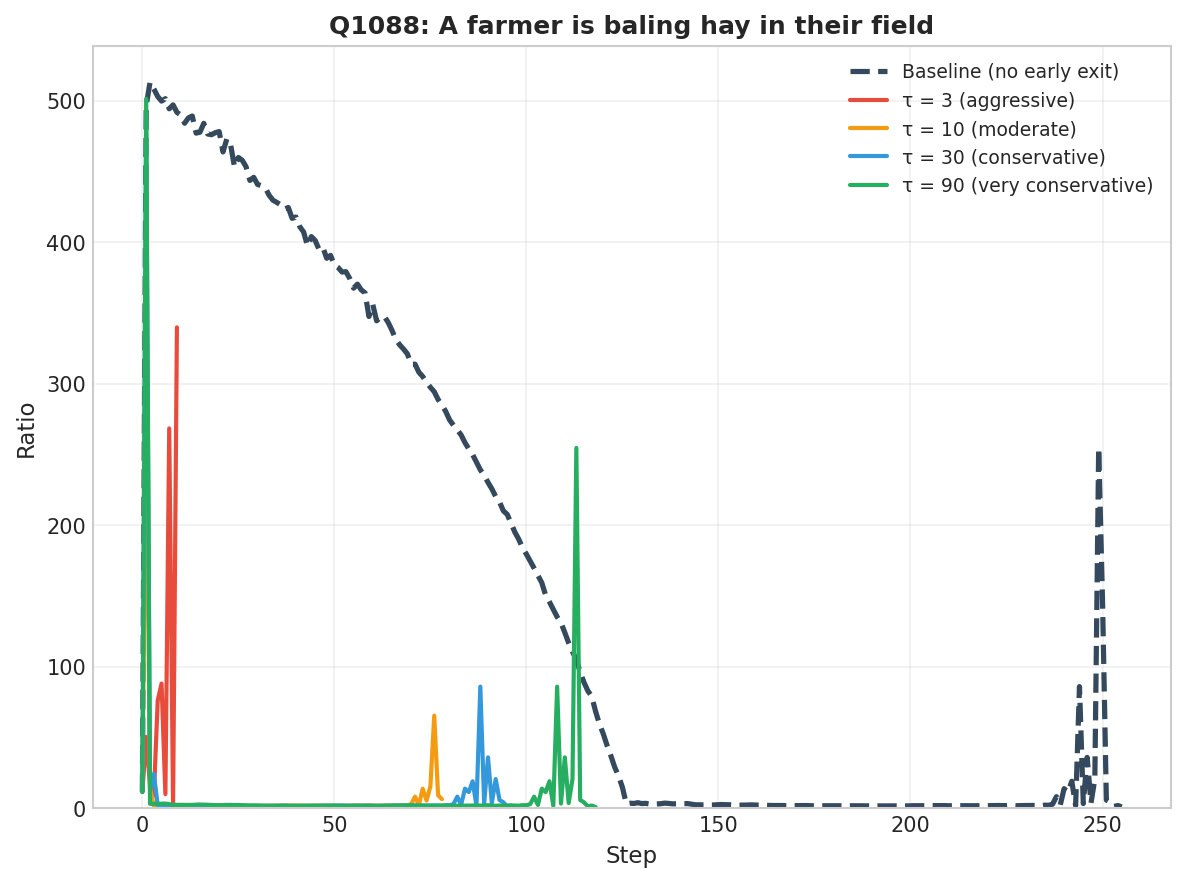}
    \caption{Confidence ratio dynamics for a GSM8K sample. The baseline (dashed) shows a smooth decline as the sampler gradually unmasks tokens. \JOT with $\tau = 90$ rapidly finalizes high-confidence tokens, then exhibits spikes during the reasoning phase before converging to a similar pattern as the baseline.}
    \label{fig:dynamics_1}
\end{figure}

\begin{figure}[t]
    \centering
    \includegraphics[width=0.8\columnwidth]{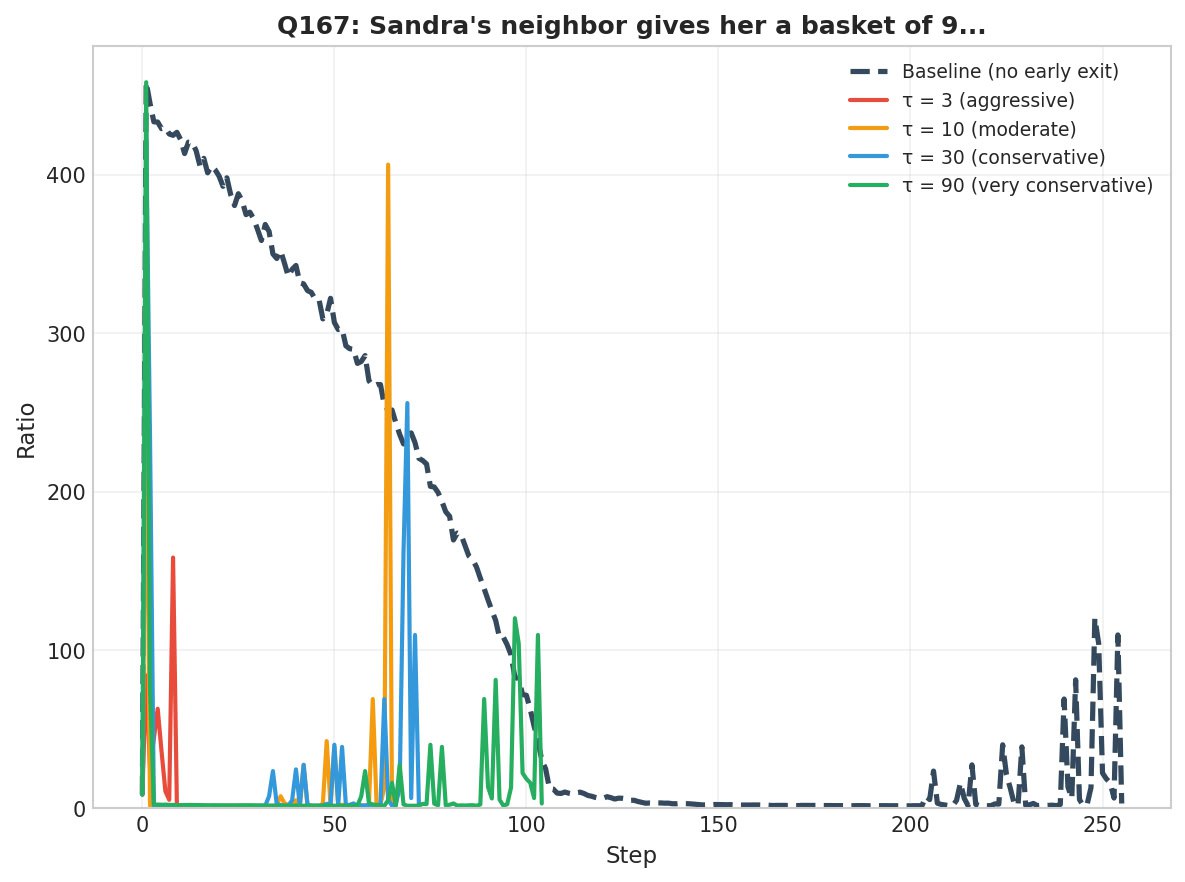}
    \caption{Confidence ratio dynamics for a second GSM8K sample. Conservative thresholds ($\tau \geq 30$) allow the model to reach the reasoning phase and produce answer patterns similar to the baseline, while aggressive thresholds ($\tau = 3$) exit prematurely.}
    \label{fig:dynamics_2}
\end{figure}

To understand how \JOT affects the generation process, we visualize the mean confidence ratio across diffusion steps for two GSM8K samples (\cref{fig:dynamics_1,fig:dynamics_2}). The baseline decoding exhibits a characteristic pattern: confidence starts very high and declines smoothly over steps. This occurs because the standard transfer schedule forces the model to unmask tokens gradually, even when many positions are already highly confident. The model effectively spends early steps filling in ``obvious'' context tokens that it could predict immediately, before entering a lower-confidence reasoning phase where it works on the actual answer.

\JOT fundamentally changes this dynamic. By allowing early exit for confident predictions, the method permits the model to finalize obvious tokens instantly and proceed directly to the reasoning phase. This manifests as a rapid initial drop in mean confidence (as high-confidence positions exit), followed by sustained activity during reasoning. Crucially, with conservative thresholds ($\tau = 90$), the confidence pattern at the end of the reasoning phase closely mirrors that of the baseline, indicating that the model follows a similar computational trajectory to arrive at the answer. Aggressive thresholds ($\tau = 3$) exit too early, cutting off the reasoning process before the model can fully work through the problem. This explains the accuracy degradation reported in the threshold ablation of the main paper. More conservative thresholds ($\tau = 30, 60$) provide sufficient time for reasoning while still achieving substantial speedups by eliminating redundant refinement of already-converged positions.

\Cref{fig:dynamics_2} shows that the same pattern holds on a second sample, indicating the behavior is not specific to a single problem instance.

\subsection{Failure Mode Analysis}
\label{app:failure_modes}

To characterize when \JOT introduces errors, we compare \JOT against full decoding on Dream-7B over $n=500$ GSM8K samples, at a conservative ($\tau_{\max}=90$) and an aggressive ($\tau_{\max}=30$) setting ($\gamma=0.5$, $D=8$ in both cases).

\paragraph{Overall effect.} Full decoding scores $83.2\%$. \JOT at $\tau_{\max}=90$ scores $81.0\%$ ($-2.2$ points) at $5.69\times$ speedup; a McNemar exact test on the paired outcomes does not reject equality ($p \approx 0.21$). At $\tau_{\max}=30$, accuracy drops to $71.4\%$ ($-11.8$ points, $p < 10^{-9}$) at $7.65\times$. Final-answer agreement with the baseline is $87.4\%$ at $\tau_{\max}=90$ and $82.2\%$ at $\tau_{\max}=30$. The disagreements are not one-sided: \JOT \emph{repairs} baseline errors in $5.2\%$ and $3.0\%$ of samples, respectively.

\paragraph{Identified failure modes.} We identify two primary failure patterns in multi-step reasoning.

\emph{Single-token catastrophic errors} occur when \JOT early-commits a token at a critical branching point---such as an arithmetic operator or intermediate value---producing a plausible but incorrect reasoning chain. For example, given ``\textit{a recipe makes 32\,oz of sauce using half as many oz of tomatoes},'' baseline correctly generates $32 / 2 = 16$\,oz, while \JOT commits ``$\times\,2$'' instead of ``$/\,2$,'' yielding $32 \times 2 = 64$\,oz. The rest of the chain is internally consistent but produces a wrong final answer.

\emph{Premature answer anchoring} occurs when the model exits on an intermediate variable rather than the final asked quantity. For instance, given ``\textit{110 coins total, 30 more gold than silver---how many gold?},'' both models solve $x + (x{+}30) = 110 \Rightarrow x = 40$. Baseline continues to compute gold $= 40 + 30 = 70$, while \JOT emits the intermediate $x = 40$ as the final answer. Importantly, \JOT does not introduce entirely novel failure modes---every observed error type also occurs under full decoding, albeit at different rates.

\paragraph{Speedup--error relationship.} Stratifying the $500$ samples by achieved speedup shows that errors concentrate at \emph{low} achieved speedups. The \emph{total} error rate by achieved-speedup quintile (Q1 slowest to Q5 fastest) is $30/29/19/10/7\%$ at $\tau_{\max}=90$ and $49/34/33/16/11\%$ at $\tau_{\max}=30$; the extreme quintiles have non-overlapping $95\%$ Wilson intervals. Achieved speedup thus acts as a difficulty proxy: samples on which many tokens exit early are those the model finds easy, so fast exits are safe exits, whereas slow-exit samples are genuinely hard problems on which full decoding also errs.

\subsection{Composability with Fast-dLLM's DualCache}
\label{app:composability}

\begin{table}[t]
    \centering
    \small
    \setlength{\tabcolsep}{3pt}
    \begin{sc}
    \begin{tabular}{ll cc cc}
        \toprule
        & & \multicolumn{2}{c}{\textbf{GSM8K}} & \multicolumn{2}{c}{\textbf{HumanEval}} \\
        \cmidrule(lr){3-4} \cmidrule(lr){5-6}
        \textbf{Model} & \textbf{Method} & Score$\uparrow$ & Speed$\uparrow$ & Pass@1$\uparrow$ & Speed$\uparrow$ \\
        \midrule
        \multirow{4}{*}{Dream-7B} & Fast-dLLM & \textbf{78.4} & 4.08$\times$ & \textbf{53.7} & 7.32$\times$ \\
        & + \JOT ($\tau{=}30$) & 76.5 & \textbf{4.22}$\times$ & 50.6 & \textbf{8.47}$\times$ \\
        & + \JOT ($\tau{=}60$) & --- & --- & 48.2 & 7.72$\times$ \\
        & + \JOT ($\tau{=}90$) & --- & --- & 49.4 & 7.44$\times$ \\
        \midrule
        \multirow{4}{*}{LLaDA-8B} & Fast-dLLM & 75.1 & \textbf{3.40}$\times$ & \textbf{39.0} & 3.05$\times$ \\
        & + \JOT ($\tau{=}10$) & --- & --- & 34.8 & \textbf{5.15}$\times$ \\
        & + \JOT ($\tau{=}30$) & --- & --- & 35.4 & 3.61$\times$ \\
        & + \JOT ($\tau{=}60$) & \textbf{75.7} & 3.21$\times$ & 36.6 & 3.02$\times$ \\
        \bottomrule
    \end{tabular}
    \end{sc}
    \caption{\JOT composed with Fast-dLLM's DualCache. All rows decode under the same block-wise KV cache; ``+ \JOT'' rows run \JOT inside the cached decoder at the given $\tau_{\max}$. The $\tau$ sweep is conducted on HumanEval; the selected thresholds ($\tau{=}30$ for Dream, $\tau{=}60$ for LLaDA) are then also evaluated on GSM8K. Score / step-speedup; full-decoding references appear in \cref{tab:main_results}. \textbf{Bold} marks best per column and model.}
    \label{tab:composability}
\end{table}

\Cref{tab:composability} reports the full data behind the composability results in the main paper: a threshold sweep for \JOT operating inside Fast-dLLM's DualCache on HumanEval, and the selected configurations additionally evaluated on GSM8K. Under the cache, the sweep is flat on Dream ($\tau=30$ best in-sweep) and favors higher thresholds on LLaDA. Retuning narrows but does not close the HumanEval gap to Fast-dLLM's max-probability gate---consistent with approximate suffix logits penalizing the confidence-ratio criterion more than a max-probability one. On GSM8K, the composed system is at parity on Dream and best-scoring on LLaDA ($75.7$ vs.\ full decoding $74.5$).

\subsection{Spatial Modulation Ablation}
\label{app:spatial_ablation}

\begin{table*}[t]
    \centering
    \small
    \begin{sc}
    \setlength{\tabcolsep}{2pt}  
    \begin{tabular}{cc cc cc cc cc}
        \toprule
        & & \multicolumn{2}{c}{\textbf{GSM8K}} & \multicolumn{2}{c}{\textbf{MMLU}} & \multicolumn{2}{c}{\textbf{HellaSwag}} & \multicolumn{2}{c}{\textbf{HumanEval}} \\
        \cmidrule(lr){3-4} \cmidrule(lr){5-6} \cmidrule(lr){7-8} \cmidrule(lr){9-10}
        $\gamma$ & $D$ & Score$\uparrow$ & Speed$\uparrow$ & Score$\uparrow$ & Speed$\uparrow$ & Score$\uparrow$ & Speed$\uparrow$ & Pass@1$\uparrow$ & Speed$\uparrow$ \\
        \midrule
        \multicolumn{2}{c}{\textit{No spatial}} & \textbf{78.8} & 5.30$\times$ & 66.7 & 1.51$\times$ & \textbf{72.8} & 2.25$\times$ & 56.7 & 19.08$\times$ \\
        \midrule
        0.9 & 16 & 77.6 & \textbf{5.87}$\times$ & 66.7 & \textbf{1.66}$\times$ & 72.7 & \textbf{2.27}$\times$ & \textbf{59.1} & \textbf{21.27}$\times$ \\
        0.9 & 8 & 77.3 & \underline{5.75}$\times$ & 66.7 & \underline{1.64}$\times$ & 72.7 & \textbf{2.27}$\times$ & \underline{58.5} & \underline{20.24}$\times$ \\
        0.7 & 8 & \underline{78.5} & 5.63$\times$ & \textbf{66.8} & 1.60$\times$ & 72.7 & 2.26$\times$ & 57.9 & 19.88$\times$ \\
        0.5 & 16 & 78.2 & 5.57$\times$ & 66.7 & 1.57$\times$ & 72.7 & 2.26$\times$ & 57.9 & 19.71$\times$ \\
        0.5 & 8 & \textbf{78.8} & 5.54$\times$ & 66.7 & 1.57$\times$ & 72.7 & 2.26$\times$ & \underline{58.5} & 19.60$\times$ \\
        0.5 & 4 & \underline{78.5} & 5.52$\times$ & 66.7 & 1.57$\times$ & 72.7 & 2.26$\times$ & 57.3 & 19.76$\times$ \\
        0.3 & 8 & 78.4 & 5.52$\times$ & 66.7 & 1.54$\times$ & 72.5 & 2.26$\times$ & 56.1 & 19.80$\times$ \\
        \bottomrule
    \end{tabular}
    \end{sc}
    \caption{Ablation on spatial modulation with $(\tau_{\max}, \tau_{\min}) = (90, 1)$. Gains are relative to threshold-only baseline ($\tau = 90$, no spatial). Results on Dream-7B. \textbf{Bold} and \underline{underline} mark best and second-best per column.}
    \label{tab:spatial_ablation}
\end{table*}

\Cref{tab:spatial_ablation} reports the full spatial-modulation sweep on Dream-7B with $(\tau_{\max}, \tau_{\min}) = (90, 1)$, varying the decay rate $\gamma$ and window radius $D$; gains are relative to the threshold-only baseline ($\tau = 90$, no spatial modulation). For long-generation benchmarks, spatial modulation provides meaningful speedup gains: HumanEval improves from $19.08\times$ to $21.27\times$ with $(\gamma, D) = (0.9, 16)$. The configuration $(\gamma, D) = (0.5, 8)$ achieves a favorable balance, matching the no-spatial accuracy on GSM8K ($78.8\%$) while improving speedup from $5.30\times$ to $5.54\times$; HumanEval moves by $+1.8$ points ($58.5\%$), a difference within one standard error at $n = 164$. Larger decay rates ($\gamma = 0.9$) spread influence too broadly, causing accuracy degradation on GSM8K ($-1.5$ points at $\gamma = 0.9$, $D = 8$).

For MMLU (3 tokens) and HellaSwag (5 tokens), generation length is shorter than the window sizes $D$ we consider, so the spatial window effectively spans the entire sequence regardless of $D$ and results for $D = 8$ and $D = 16$ coincide with $D = 4$. This is why both benchmarks are essentially flat across the sweep. We therefore select $(\gamma, D) = (0.5, 8)$ as the default spatial configuration used in the main paper.

\section{Proofs}
\label{app:proofs}
 
\begin{lemma}[Order does not matter]
\label{lem:order}
Suppose a decoder repeatedly (i) selects a masked position by an arbitrary rule that may depend on the revealed context and on the model's predictions, and (ii) commits it by sampling from the true conditional $\mathbb{P}(x^i = \cdot \mid \mathcal{F}_n)$. Then the completed sequence is an exact sample from the true joint distribution.
\end{lemma}
 
\begin{proof}
A decoding trace is a sequence of (position, value) pairs. Its probability is the product, over steps, of the selection probability times the conditional probability of the committed value. Fix a completed sequence $x$ and sum over all orders consistent with it. By the chain rule, the product of conditionals equals $p(x)$ for \emph{every} order, so it factors out of the sum; the remaining sum of selection probabilities over all complete orders equals $1$, since at each step the selection rule is a probability distribution over the masked positions. Hence $\mathbb{P}(x) = p(x)$.
\end{proof}
 
\begin{proof}[Proof of \cref{prop:flip}]
Work conditionally on $\mathcal{F}_n$ and on the event $\{x^i \in \{\hat v, u\}\}$, and define for $m \ge n$
\[
q_m \;=\; \mathbb{P}\!\left(x^i = u \,\middle|\, \mathcal{F}_m,\; x^i \in \{\hat v, u\}\right)
\;=\; \frac{M_m(u)}{M_m(u) + M_m(\hat v)}.
\]
As the conditional probability of a fixed event given growing context, $(q_m)$ is a nonnegative martingale. At commit time, the calibration assumption identifies model and true probabilities, so $q_n = p^i_2 / (p^i_1 + p^i_2) \le 1/(1+\tau)$, using $p^i_1 \ge \tau\, p^i_2$ (which follows from $r^i \ge \tau$ since $p^i_1/p^i_2 \ge r^i$). A leading-pair reversal at time $m$ is exactly the event $\{q_m \ge \tfrac{1}{2}\}$. Ville's inequality for nonnegative martingales \citep{ville1939etude} gives
\[
\mathbb{P}\!\left(\exists\, m \ge n:\; q_m \ge \tfrac{1}{2}\right) \;\le\; \frac{q_n}{1/2} \;\le\; \frac{2}{1+\tau}.
\]
The budget statement follows by summing over the at most $L$ committed positions; commits at ratios above $\tau$ only decrease $q_n$ and tighten the bound.
\end{proof}
 
\begin{proof}[Proof of \cref{cor:floor}]
Fix $\kappa > 0$ and split on the event $E = \{q^*_n \le 1/(1+\kappa)\}$. On $E$, the proof of \cref{prop:flip} applies with the true share in place of the model's---only the martingale property of the true posterior is used, which holds without the calibration assumption---giving a flip probability of at most $2/(1+\kappa)$. On the complement of $E$, bound the flip probability by $1$; intersected with the commit condition $\{r^i \ge \tau\}$, this contributes at most $\varepsilon_{\mathrm{cal}}(\tau, \kappa)$. Summing the two cases gives \cref{eq:floor}.
\end{proof}

\end{document}